\documentclass{article}
\usepackage{fullpage,amssymb}
\usepackage{graphicx}
\usepackage{url}
\usepackage{listings}

\begin{document}
\sloppy

\def\dx{\mathrm{d}x}
\def\dh{\mathrm{d}h}

\def\erf{\mathrm{erf}}
\def\calN{\mathcal{N}}
\def\TV{\mathrm{TV}}
\def\KL{\mathrm{KL}}
\def\even{\mathrm{even}}

\def\inner#1#2{\langle #1 | #2 \rangle}

\def\eps{\epsilon}
\def\calF{\mathcal{F}}
\def\calR{\mathcal{R}}
\def\calE{\mathcal{E}}
\def\calX{\mathcal{X}}
\def\Pr{\mathrm{Pr}}
\def\supp{\mathrm{supp}}
\def\map{\mathrm{MAP}}
\def\TV{\mathrm{TV}}
\def\dnu{\mathrm{d}\nu}
\def\dmu{\mathrm{d}\mu}
\def\dt{\mathrm{d}t}

\def\Bi{\mathrm{Bi}}
\def\argmax{\mathrm{argmax}}
\def\argmin{\mathrm{argmin}}
\def\ceil#1{ {\lceil {#1} \rceil}  }
\def\KL{\mathrm{KL}}
\def\mattwo#1{\left[\begin{array}{ll}#1\end{array}\right]}
\def\supp{\mathrm{supp}}

\newtheorem{definition}{Definition}
\newtheorem{lemma}{Lemma}
\newtheorem{theorem}{Theorem}
\newtheorem{corollary}{Corollary}
\newtheorem{remark}{Remark}

\newenvironment{proof}[1][Proof]{\begin{trivlist}
\item[\hskip \labelsep {\bfseries #1}]}{\end{trivlist}}

\sloppy
 
\date{}

\title{Generalized Bhattacharyya and Chernoff upper bounds on Bayes error using quasi-arithmetic means}

\author{Frank Nielsen\thanks{Sony Computer Science Laboratories, Inc.
3-14-13 Higashi Gotanda, Shinagawa-Ku,
Tokyo 141-0022, Japan. {\tt Frank.Nielsen@acm.org}\ {\tt http://www.sonycsl.co.jp/person/nielsen/}. 
 Joseph-Louis Lagrange laboratory,
 Univ. Nice Sophia-Antipolis, CNRS, OCA,
  France.}\ \thanks{Accepted manuscript to appear in Pattern Recognition Letters (\protect\url{10.1016/j.patrec.2014.01.002}). \protect\url{http://www.journals.elsevier.com/pattern-recognition-letters/}. See \protect\url{http://www.journals.elsevier.com/theoretical-computer-science/news/manuscript-posting-in-arxiv/}}}

\maketitle 

\begin{abstract}
Bayesian classification labels observations based on given prior information, namely class-{\em a priori} and class-conditional probabilities.
Bayes' risk is the minimum expected classification cost that is achieved by the Bayes' test, the optimal decision rule. 
When no cost incurs for correct classification and unit cost is charged for misclassification, Bayes' test reduces to the maximum {\em a posteriori} decision rule, and Bayes risk  simplifies to Bayes' error, the probability of error. 
Since calculating this probability of error is  often intractable, several techniques have been devised to bound it with closed-form formula, introducing thereby measures of similarity and divergence between  distributions like the Bhattacharyya coefficient and its associated Bhattacharyya distance.
The  Bhattacharyya upper bound can  further be tightened using the Chernoff information that relies on the notion of best error exponent. 
In this paper, we first express Bayes' risk using the total variation distance on scaled distributions.
We then elucidate and extend the Bhattacharyya and the Chernoff upper bound mechanisms using generalized weighted means.
 We provide
 as a byproduct novel notions of statistical divergences and affinity coefficients.
We illustrate our technique by deriving new upper bounds for the univariate Cauchy and the multivariate $t$-distributions, and
show experimentally that those bounds are not too distant to the computationally intractable Bayes' error.
\end{abstract}


\def\dX{\mathrm{d}X}

\section{Introduction: Hypothesis testing, divergences and affinities}

\subsection{Hypothesis testing}

Consider the following fundamental {\em binary hypothesis testing} problem\footnote{We refer the reader to the textbooks~\cite{ct-1991,itincvpr-2009} for an information-theoretic background based on the method of types, and to the textbook~\cite{Fukunaga-1990} for the Bayesian setting often met in pattern recognition.}:
Let $X_1, ..., X_n$ be $n$ identically and independently distributed (IID) random variables following distribution $Q$ with support $\mathcal{X}$.
We consider two (simple) hypotheses:
\begin{eqnarray}
H_1 &:& Q\sim P_1 (\mbox{null hypothesis}),\\
H_2 &:& Q\sim P_2 (\mbox{alternative hypothesis})
\end{eqnarray}
and we design a {\em test} $g(X_1, ..., X_n): \mathcal{X}^n \rightarrow \{1,2\}$ to decide which hypothesis to select. 
The decision region $R_1\subseteq \mathcal{X}^n$ corresponds to the set of sequences $X^n=(X_1, ..., X_n)$ mapped to $H_1$, and the decision region $R_2=R_1^c$ is the complementary region.

To  illustrate this setting,  consider for example the task of distinguishing a texture~\cite{BayesianInference-2000}, modeled by edglets\footnote{An edgelet is a small line segment with slope quantized to take $d$ possible directions.} $X^n$ centered at 2D image lattice positions, from two textures $T_1$ and $T_2$, given by their respective edgelet probability distributions $P_1$ and $P_2$ (assuming the IID hypothesis). 
In practice, we {\em observe} a {\em texture sample}, that is a data set  $x^n=(x_1, ..., x_n)$ sampled from $X^n$.

There are two kinds of error~\cite{ct-1991} associated with any test:
\begin{itemize}
\item Type I error (misclassification when the true hypothesis is $H_1$): $\epsilon_1(n)= \Pr(g(X_1, ..., X_n)=2|H_1)$, and
\item Type II error (misclassification when the true hypothesis is $H_2$): $\epsilon_2(n)= \Pr(g(X_1, ..., X_n)=1|H_2)$
\end{itemize}
In target/noise detection theory~\cite{ct-1991}, a test is called a {\em detector}, and those type I and type II errors are respectively called probability of {\em false alarm} and probability of {\em miss}.

There are two main approaches for hypothesis testing that have been developed in the literature~\cite{ct-1991,Fukunaga-1990}:
The first Neyman-Pearson approach  seeks to minimize the probability of miss given the probability of false alarm, without any prior information for the hypothesis.
The second Bayesian approach makes use of prior information on class-{\em a priori} and class-conditional probabilities for the hypothesis.
We concisely review the links between hypothesis testing and statistical distances between distributions for the first non-Bayesian approach in Section~\ref{sec:DivHT}, and the links between  hypothesis testing and statistical similarities between distributions for the second Bayesian approach in Section~\ref{sec:SimHT}.
 
\subsection{Statistical divergences in hypothesis testing\label{sec:DivHT}}
 
The first approach asks to  minimize the probability of miss $\eps_2$ ({\em false negative}) given the probability of false alarm $\eps_1$ ({\em false positive}):

\begin{equation}
\min_{\epsilon_2} \epsilon_1\leq \varepsilon,
\end{equation}
where $\varepsilon$ is a prescribed error threshold.  
In the literature, the {\em significance level} (or size) of a test is $\epsilon_1$ and the {\em power} of a test is $1-\eps_2$.
Thus we seek to maximize the {\em power of a test} given a prescribed significance level.
A key result is the Neyman-Pearson lemma~\cite{ct-1991} which states the optimality of the {\em Likelihood Ratio Test} (LRT) (or equivalently its log-likelihood ratio):
\begin{equation}\label{eq:lrt}
\Lambda(X_1, ..., X_n) = \log \frac{P_1(X_1, ..., X_n)}{P_2(X_1, ..., X_n)} = \sum_{i=1}^n \log \frac{P_1(X_i)}{P_2(X_i)} \leq \lambda,
\end{equation}
to reject $H_1$ in favor of $H_2$ (with $\Pr(\Lambda(X_1, ..., X_n)\leq\lambda|H_1)=\epsilon_1$ and $\lambda=\lambda(\varepsilon)$). 
Note that the larger the log-likelihood ratio, the more probable the sequence $X^n$ comes from $P_1$ (and the more likely hypothesis $H_1$).
From the IDD assumption, we better analyze sequences via the {\em method of types}~\cite{BayesianInference-2000,itincvpr-2009}:
The type $h(x^n)$ is the empirical probability distribution of elements of $\mathcal{X}$ (say, a discrete alphabet with $d$ letters) met in $x^n$.
That is, for  alphabet $\mathcal{X}=\{E_1, ..., E_d\}$, the type $h(x^n)$ of a sample sequence is the frequency empirical histogram of elements: $h(x^n)=(h_1(x^n), ..., h_d(x^n))$, where $h_i(x^n)=\frac{1}{n}\sum_{j=1}^n \delta_{x_j,E_i}=\frac{\#\{x_j= E_i\ |\ j\in\{1, ..., n\}\}}{n}$.
Observe that although there are $d^n$ distinct sequences of length $n$ (that is, exponential in $n$), there is only a polynomial number of types (bounded by $(n+1)^d$ since each of the $d$ elements $E_i$ of $\mathcal{X}$ has counting number $n_i=n h_i(X^n)$, an integer between $0$ and $n$). 

The log-likelihood ratio of Eq.~\ref{eq:lrt} rewrites as:
\begin{equation}
\Lambda(x_1, ..., x_n) = \sum_{i=1}^n \log \frac{P_1(x_i)}{P_2(x_i)} =  \sum_{j=1}^d n h_j(x^n) \log \frac{P_1(E_j)}{P_2(E_j)}. 
\end{equation}
Observe that the log-likelihood ratio can be conveniently written as an inner product between two histograms $h(x^n)$ and 
$A_{12}=(\log  \frac{P_1(E_1)}{P_2(E_1)}, ..., \log  \frac{P_1(E_d)}{P_2(E_d)})$: $\Lambda(x_1, ..., x_n) =  n \inner{h(x^n)}{A_{12}}$.
In~\cite{BayesianInference-2000}, the inner product $ \inner{h(x^n)}{A_{12}}$ is called the {\em reward} of sequence sample $x^n$.

The {\em expected average log-likelihood ratios} with respect to $P_1$ and $P_2$ are:
\begin{eqnarray}
\frac{1}{n} E_{P_1}[\Lambda(X_1, ..., X_n)] &=&    \sum_{j=1}^d   P_1(E_j) \log \frac{P_1(E_j)}{P_2(E_j)}=\KL(P_1:P_2),\\
\frac{1}{n} E_{P_2}[\Lambda(X_1, ..., X_n)] &=&   \sum_{j=1}^d   P_2(E_j) \log \frac{P_1(E_j)}{P_2(E_j)}=-\KL(P_2:P_1),
\end{eqnarray}
where $\KL(P_1:P_2)= \sum_{i=1}^d P_1(E_i) \log\frac{P_1(E_i)}{P_2(E_i)}$ denotes the {\em Kullback-Leibler divergence} between distributions $P_1$ and $P_2$.
The difference between the two expected average log-likelihood ratio is the {\em Jeffreys divergence} $J(P_1,P_2)=\KL(P_1:P_2)+\KL(P_2:P_1)$, that symmetrizes the Kullback-Leibler divergence. 
It follows that those Kullback-Leibler and Jeffreys measures can be interpreted as {\em measures of separability}, that is distances between distributions.
Note that the KL distance is asymmetric: $\KL(P_1:P_2)\not = \KL(P_2:P_1)$.
The KL and J distances are not metric because they  violate the triangular inequality~\cite{ct-1991}. 

The probability that a sequence sample $x_1^n$ from $P_1$ has lower reward than a sequence sample $x_2^n$ from $P_2$ is bounded by (see~\cite{BayesianInference-2000}, Theorem~2):

\begin{equation}
  (n+1)^{-d^2} 2^{-n B(P_1,P_2)} \leq \Pr( \inner{h(x_1^n)}{A_{12}} \leq \inner{h(x_2^n)}{A_{12}} ) \leq  (n+1)^{d^2} 2^{-n B(P_1,P_2)},
\end{equation}
where $B(P_1,P_2)$ denotes the {\em Bhattacharrya divergence}:

\begin{equation}
B(P_1,P_2) = -\log \sum_{j=1}^d \sqrt{P_1(E_j)}  \sqrt{P_2(E_j)}.
\end{equation}

Although the Neyman-Pearson lemma~\cite{ct-1991} characterizes the optimal decision test, it does not specify the threshold $\lambda$.
The false alarm error $\epsilon_1$ and miss error $\epsilon_2$  probabilities decay exponentially as the sample size $n$ increases (see~\cite{BayesianInference-2000}, Theorem~1).
Thus in the asymptotic regime, we are rather interested in characterizing the {\em error exponents}  defined as the rate of that exponential decay:

\begin{equation}
\alpha =\lim_{n\rightarrow \infty} -\frac{1}{n} \log_2 \epsilon_1(n), \quad\beta =\lim_{n\rightarrow \infty} -\frac{1}{n} \log_2 \epsilon_2(n),
\end{equation}
where $\epsilon_1(n)=\Pr(g(X^n)=2|H_1)$ and $\epsilon_2(n)=\Pr(g(X^n)=1|H_2)$ (and $\epsilon_1(n)\approx 2^{-n\alpha}$ and $\epsilon_2(n)\approx 2^{-n\beta}$).
It turns out that when minimizing the asymptotic rate of {\em misclassification error} $P_e=\eps_1+\eps_2$, called the probability of error, the optimal threshold is $\lambda=0$, and the
error rate is the Chernoff information\cite{ct-1991,ChernoffGeometry-2013}: 

\begin{equation}
C(P_1,P_2) =   \min_{\alpha\in [0,1]}  B_\alpha(P_1:P_2), 
\end{equation}
where $B_\alpha$ denotes the {\em skewed Bhattacharrya divergence}:

\begin{equation}
B_\alpha(P_1:P_2) = -\log \sum_{j=1}^d P_1(E_j)^{\alpha}  P_2(E_j)^{1-\alpha},
\end{equation}
generalizing the Bhattacharrya divergence: $B(P_1,P_2)=B_{\frac{1}{2}}(P_1:P_2)$.

Those notions of statistical divergences can be extended to continuous distributions by replacing the discrete sum by an integral  (and interpreted $\mathcal{X}$ as a continous alphabet). 

We now consider the Bayesian paradigm in hypothesis testing, and show how to bound the probability of misclassification error using statistical similarity measures.

\subsection{Statistical similarities in hypothesis testing\label{sec:SimHT}}

The Bayesian framework of hypothesis testing assumes that we are given prior beliefs over the probabilities of the two hypothesis, and we seek to minimize the {\em expected probability of error} (also called {\em error probability}): $P_e=\epsilon_1 \Pr(H_1) + \epsilon_2 \Pr(H_2)$.
In this setting, both the  {\em class a priori} ($w_i>0$) and the {\em class conditional} probabilities ($p_i$) are known beforehand (or estimated from a training labeled sample~\cite{Fukunaga-1990}). 
Let $q_1$ and $q_2$ be the {\em a posteriori} probabilities derived from Bayes theorem:
\begin{equation}
q_i(x) = \frac{w_i p_i(x)}{p(x)}, 
\end{equation}
where $p(x)$ is the mixture density $p(x)=w_1 p_1(x)+ w_2 p_2(x)$ (and $w_1+w_2=1$).
Let $C=[c_{ij}]$ be the  $2\times 2$ {\em design matrix},  with $c_{ij}$  denoting the {\em cost} of deciding  $x\in C_i$ when $x\in C_j$, with $1\leq i,j\leq 2$.
Furthermore, denote by 
\begin{equation}
r_1(x)=c_{11}q_1(x)+c_{12}q_2(x),
\end{equation}
and  
\begin{equation}
r_2(x)=c_{21}q_1(x)+c_{22}q_2(x),
\end{equation}
the respective {\em conditional costs} of deciding $x\in C_i$, for $i\in\{1,2\}$.
To classify $x$, consider the decision rule: 
\begin{equation}\label{eq:dr}
r_1(x)  \mathop{\lesseqgtr}^{C_1}_{C_2} r_2(x).
\end{equation}

The conditional cost of this decision rule is:

\begin{equation}
r(x)=\min(r_1(x),r_2(x)).
\end{equation}

{\em Bayes error} $B_e$ (see~\cite{Fukunaga-1990}, p. 57)  is defined as the {\em expected} cost of this decision rule: 

\begin{eqnarray}
B_e&=& E_p[r(x)],\\
&=& \int p(x) \min(r_1(x),r_2(x))\dx,  \label{eq:bemin}\\
&=&
\int_{R_1} (c_{11}w_1p_1(x)+c_{12}w_2p_2(x)) \dx
+
\int_{R_2} (c_{21}w_1p_1(x)+c_{22}w_2p_2(x)) \dx,
\end{eqnarray}
where 
$R_1=\{x\ |\ r_1(x)\leq r_2(x) \}$ and $R_2=\{x\ |\ r_2(x)\leq r_1(x) \}$ are the {\em decision regions} induced by the decision rule of Eq.~\ref{eq:dr}.
Thus Bayes test for minimum cost writes as:

\begin{equation}
(c_{12}-c_{22}) w_2 p_2(x) \mathop{\gtreqless}^{C_1}_{C_2}   (c_{21}-c_{11})w_1p_1(x)
\end{equation}
or equivalently:

\begin{equation}
\frac{p_1(x)}{p_2(x)}  \mathop{\lesseqgtr}^{C_1}_{C_2}  \frac{w_2(c_{12}-c_{22})}{w_1(c_{21}-c_{11})}.
\end{equation}

The term $l(x)=\frac{p_1(x)}{p_2(x)}$ is called the {\em likelihood ratio}. 
It is equivalent and often mathematically simpler to consider the test
using the log-likelihood ratio (e.g., think of the multivariate Gaussian class-conditional probabilities):

\begin{equation}
\log p_1(x)-\log p_2(x) \mathop{\lesseqgtr}^{C_1}_{C_2} \log \frac{w_2(c_{12}-c_{22})}{w_1(c_{21}-c_{11})}
\end{equation}

The function $h(x)=\log p_1(x)-\log p_2(x)-\frac{w_2(c_{12}-c_{22})}{w_1(c_{21}-c_{11})}$ is called the {\em discriminant function}.

For {\em symmetrical cost} function $c_{21}-c_{11}=c_{12}-c_{22}$, 
the expected cost is called the {\em probability of error} $P_e$.
For the probability of error, we do not incur a cost for correctly classifying and pay a unit cost for misclassification.
The design matrix for the probability of error is therefore:
\begin{equation}
C=\left[\begin{array}{ll}0 & 1\\ 1 & 0\end{array}\right].
\end{equation} 
As stated earlier, the probability of error can be decomposed as the sum of two misclassification costs:

\begin{equation}
P_e =  w_1 \epsilon_1 + w_2 \epsilon_2,
\end{equation}
with
\begin{equation}
\epsilon_1 = \int_{R_2} p_1(x)\dx,\qquad \epsilon_2=\int_{R_1} p_2(x)\dx,
\end{equation}
with $R_1=\{w_2p_2(x) \leq w_1p_1(x)  \}$ and $R_2=\{x \ |\ w_1p_1(x) \leq w_2p_2(x) \}$.

In practice, Bayes error and the probability of error are quite tricky to calculate as we need to compute integrals on decision regions $R_1$ and $R_2$.
Even if those domains can be expressed simply, say for  Gaussians, it is often intractable to compute analytically those integrals (e.g., for multivariate class-conditional probabilities). See the Appendix for a review of formula when class-conditional distributions are Gaussians.
Therefore, we need to set good lower and upper bounds to characterize Bayes error $B_e$ (or the  probability of error $P_e$).

A first upper bound on the probability of error $P_e$ is the Bhattacharrya bound~\cite{ProbabilityErrorChernoff-1970}:

\begin{equation}\label{eq:pebhat}
P_e \leq \sqrt{w_1w_2}\times  \rho(p_1,p_2),
\end{equation}  
with $\rho(p_1,p_2)=\int \sqrt{p_1(x)p_2(x)} \dx$ denoting the {\em Bhattacharrya coefficient}.
This first upper bound was tightened by Chernoff~\cite{Chernoff-1952} as follows: We have

\begin{equation}\label{eq:pechernoffalpha}
P_e \leq w_1^{\alpha}w_2^{1-\alpha}\times  \rho_\alpha(p_1,p_2),
\end{equation}  
with $ \rho_\alpha(p_1,p_2)=\int (p_1(x))^{\alpha} (p_2(x))^{1-\alpha}   \dx$, the {\em $\alpha$-Chernoff coefficient} defined for $\alpha\in(0,1)$.
Therefore the tightest upper bound is:

\begin{equation}\label{eq:pechernoff}
P_e \leq w_1^{\alpha^*}w_2^{1-\alpha^*}\times  \rho_*(p_1,p_2),
\end{equation}  
with $\rho_*$ the {\em Chernoff coefficient} obtained from the following minimization problem:

\begin{equation}
\rho_*(p_1,p_2)=\min_{\alpha\in [0,1]} \rho_\alpha(p_1,p_2),
\end{equation}
where $\alpha^*\in[0,1]$ denotes the optimal value that minimizes $\rho_\alpha(p_1,p_2)$.  
Note that the Bhattacharyya  coefficient is a particular case of the Chernoff $\alpha$-coefficient (obtained for $\alpha=\frac{1}{2}$): $\rho=\rho_{\frac{1}{2}}$.

The Bhattacharyya, $\alpha$-Chernoff and Chernoff coefficients $\rho, \rho_\alpha$ and $\rho_*=\rho_{\alpha^*}$ can be interpreted a 
{\em similarity measures} between distributions defined by measuring the degree of overlap of their densities.
Those coefficients are also called affinities.

\begin{remark}[Affinity coefficients and divergences]
The Chernoff $\alpha$-coefficient $\rho_\alpha$, the Chernoff coefficient $\rho_*$  and the Bhattacharrya coefficient $\rho$ ($\rho=\rho_{\frac{1}{2}}$) provide upper bounds on $P_e$: $0<P_e\leq \frac{1}{2}\rho_* \leq \frac{1}{2}\rho \leq \frac{1}{2}$ 
(for $w_1=w_2=\frac{1}{2}$, see Eq.~\ref{eq:pechernoff}, and Eq.~\ref{eq:pebhat}). 
We can transform any affinity coefficient $0<A\leq 1$ into a corresponding divergence by applying a monotonously increasing function $f$ to $\frac{1}{A}$ (with $\frac{1}{A} \in [1,\infty)$). By choosing $f(x)=\log(x)$, we end-up with the traditional divergences between statistical distributions:
$\alpha$-Chernoff divergence, Chernoff divergence (also called Chernoff information), and Bhattacharrya divergence. 
\end{remark}

Recall that those upper bounds are useful if they can be computed easily from closed-form formula (which is not the case of $B_e$ nor $P_e$).

\subsection{Closed-form Bhattacharrya/Chernoff coefficients for exponential families}

Many usual distributions like Gaussians, Poisson, Dirichlet or Gamma/Beta, etc. distributions are exponential families in disguise~\cite{Kailath-1967,2011-brbhat} for which the {\em skewed affinity coefficient} $\rho_\alpha$ (i.e., similarity distance within $[0,1]$)
can be computed in closed-form.
An exponential family is a family $\mathcal{F}$ of distributions: 
\begin{equation}
\mathcal{F} = \{  p(x;\theta)=\exp (x^\top \theta-F(\theta))\ |\ \theta\in\Theta \},
\end{equation}
indexed by a parameter $\theta\in \Theta$. 
Space $\Theta$ is the parameter domain, called the natural parameter space~\cite{2011-brbhat}.
$F$ is a strictly convex and differentiable convex function called the log-normalizer.
For example, the multivariate normal (MVN) distributions of mean $\mu$ and covariance matrix $\Sigma$ are exponential families for parameter $\theta=(\theta_1=\Sigma^{-1}\mu,\theta_2=\frac{1}{2}\Sigma^{-1})\in\mathbb{R}^d\times S_d+$ (where $S_d^+$ denotes the space of symmetric positive definite $d\times d$ matrices). The function $F$ (called log-normalizer) expressed in the natural coordinate system is~\cite{ChernoffGeometry-2013}:

\begin{equation}
F_{\mathrm{MVN}}(\theta_1,\theta_2) = \frac{1}{2} \theta_1^\top \theta_2^{-1}\theta_1 -\log |\theta_2|,
\end{equation}
where $|\cdot|$ denotes the determinant for a matrix operand.

Wlog., let the class-conditional probabilities $p_1$ and $p_2$ belong to the same exponential family, then we have~\cite{Kailath-1967,2011-brbhat}:

\begin{equation}
\rho_\alpha(p_1,p_2)= e^{-J_F^{(\alpha)}(\theta_1,\theta_2)},
\end{equation}
where $J_F^{(\alpha)}(\theta_1,\theta_2)$ denote the {\em Jensen skewed divergence}~\cite{2011-brbhat}:

\begin{equation}
J_F^{(\alpha)}(\theta_1,\theta_2)=\alpha F(\theta_1)+(1-\alpha)F(\theta_2)-F(\alpha \theta_1+(1-\alpha)\theta_2)\geq 0.
\end{equation}

\def\MVN{\mathrm{MVN}}

For example, let us consider the multivariate Gaussian family. Then, we get the Chernoff $\alpha$-coefficient:

\begin{equation}\label{eq:rhomvn}
\rho_\alpha^{\MVN}(p_1,p_2)= \frac{|\Sigma_1|^\frac{\alpha}{2} |\Sigma_2|^{\frac{1-\alpha}{2}}}{|\alpha\Sigma_1+(1-\alpha)\Sigma_2|^{\frac{1}{2}}}
\exp\left(-\frac{\alpha(1-\alpha)}{2}\Delta\mu^\top  (\alpha\Sigma_1+(1-\alpha)\Sigma_2)\Delta\mu\right),
\end{equation}
with $\Delta\mu=\mu_2-\mu_1$.
Therefore the Chernoff $\alpha$-divergence, $D_\alpha^{\MVN}(p_1,p_2) =-\log \rho_\alpha^{\MVN}(p_1,p_2)$:
\begin{eqnarray}
 D_\alpha^{\MVN}(p_1,p_2) = \frac{1}{2} \log \frac{|\alpha\Sigma_1+(1-\alpha)\Sigma_2|}{|\Sigma_1|^\alpha |\Sigma_2|^{1-\alpha}}
 +\frac{\alpha(1-\alpha)}{2} \Delta\mu^\top (\alpha\Sigma_1+(1-\alpha)\Sigma_2)\Delta\mu
\end{eqnarray}
Setting $\alpha=\frac{1}{2}$, we get the {\em Bhattacharrya divergence}.

In general, we do not have a closed-form solution for finding the optimal $\alpha^*$ yielding the Chernoff coefficient/information.
Nevertheless, the optimal value $\alpha^*$ of $\alpha$ can be exactly characterized~\cite{ChernoffGeometry-2013} using the differential-geometric structure of the statistical manifold of the class-conditional distributions, and yields a fast algorithm to arbitrarily finely approximate Chernoff information $\rho_*$ for members of the same exponential family.

\subsection{Outline}
This paper is organized as follows: In Section~\ref{sec:TV} we show how Bayes error is related to the total variation distance.
Section~\ref{sec:ChernoffUB} presents our generalization of Bhattacharrya and Chernoff upper bounds relying on generalized weighted means.
Section~\ref{sec:apps} illustrates several applications of the technique yielding novel upper bounds for various distributions that do not belong to the exponential families, and Section~\ref{sec:tightness} studies the tightness of those bounds.
Finally, Section~\ref{sec:concl} concludes this work.
Appendix~\ref{app:mvn} recalls the Bayes error  formula when class-conditional distributions belong to the univariate or the multivariate Gaussian families.

\section{Bayes error and the total variation distance: An identity\label{sec:TV}}

Recall Bayes error expression of Eq.~\ref{eq:bemin}:

\begin{eqnarray}
B_e&=& \int p(x) \min(r_1(x),r_2(x))\dx\\
&=&\int p(x) \min(c_{11}q_1(x)+c_{12}q_2(x) , c_{21}q_1(x)+c_{22}q_2(x))\dx 
\end{eqnarray}
with $q_1(x)=\frac{w_1p_1(x)}{p(x)}$ and $q_2(x)=\frac{w_2p_2(x)}{p(x)}$ the {\it a posteriori probabilities}.
Using the mathematical rewriting trick: 
\begin{equation}
\min(a,b)=\frac{a+b}{2}-\frac{1}{2}|b-a|,
\end{equation}
we get:

\begin{equation}
B_e =\frac{1}{2} \int \left( 
a_1 p_1(x) + a_2 p_2(x) - |a_2 p_2(x) -  a_1 p_1(x)|
\right)  \dx,
\end{equation}
where $a_1=w_1(c_{11}+c_{21})$ and $a_2=w_2(c_{12}+c_{22})$.
Finally, using the fact that $\int p_1(x)\dx=\int p_2(x)\dx=1$, we end up with:
\begin{equation}
B_e=\frac{a_1+a_2}{2} - \TV(a_1p_1, a_2p_2),
\end{equation}
where 
\begin{equation}
\TV(p,q)=\frac{1}{2}\int |p(x)-q(x)|\dx,
\end{equation}
 denotes the {\em total variation distance} extended to {\em positive distributions} (i.e., not necessarily normalized probability distributions).
In particular, for the probability of error, we have $a_1=w_1$ and $a_2=w_2$ (with $a_1+a_2=1$) and get:

\begin{equation}
P_e = \frac{1}{2}-\TV(w_1 p_1, w_2 p_2).
\end{equation}
Note that $\TV(w_1 p_1,w_2 p_2)=w_1\TV(p_1,\frac{w_2}{w_1}p_2)=w_2\TV(\frac{w_1}{w_2}p_1,p_2)$.
Therefore in the special case $w_1=w_2=\frac{1}{2}$, we get the probability of error related to the total variation distance (a metric) on probability distributions by:

\begin{equation}\label{eq:petv}
P_e = \frac{1}{2} (1-\TV(p_1,p_2)).
\end{equation}
For sanity check, notice that when $p_1=p_2$ (undistinguishable distributions), we have $\TV(p_1,p_2)=0$ and $P_e=\frac{1}{2}$.
Clearly, $0\leq \TV(p_1,p_2)\leq 1$ and $0\leq P_e\leq \frac{1}{2}$.
We summarize the result in the following theorem:

\begin{theorem}
The Bayes error $B_e$ for the cost design matrix $C=[c_{ij}]$ is related to the total variation metric distance $\TV(p,q)=\frac{1}{2}\int |p(x)-q(x)|\dx$ by $B_e=\frac{a_1+a_2}{2} - \TV(a_1p_1, a_2p_2)$ with $a_1=w_1(c_{11}+c_{21})$ and $a_2=w_2(c_{12}+c_{22})$.
The identity simplifies for probability of error $P_e$ to $P_e=\frac{1}{2}-\TV(w_1 p_1, w_2 p_2)$. 
\end{theorem}
 
Thus if we can compute the total variation distance of class-conditional probabilities $p_1$ and $p_2$, we can deduce the probability of error, and vice-versa: 
\begin{equation}
\TV(p_1,p_2)=1-2P_e\geq 0.
\end{equation}

\section{Upper bounds with  generalized means\label{sec:ChernoffUB}}

Without loss of generality, consider the probability of error $P_e$:

\begin{equation}
P_e = \int \min(w_1 p_1(x), w_2 p_2(x)) \dx = S(w_1 p_1,w_2 p_2).
\end{equation}
The probability of error can be interpreted as a {\em similarity measure} $S(w_1 p_1,w_2 p_2)$, extending the definition of {\em histogram intersection}~\cite{KernelHistogramIntersection-2009} to continuous domains. 

Chernoff~\cite{Chernoff-1952} made use of the following mathematical trick: 

\begin{equation}
\min(a,b)\leq a^\alpha b^{1-\alpha}, \forall a,b>0
\end{equation}
to define the {\em Chernoff information} upper bounding $P_e$:

\begin{equation}
P_e \leq w_1^{\alpha^*}w_2^{1-\alpha^*} \rho_*(p_1,p_2),
\end{equation}  
with
$$
\rho_*(p_1,p_2)=\min_{\alpha\in [0,1]} \rho_\alpha(p_1,p_2), \qquad \rho_\alpha(p_1,p_2)=\int (p_1(x))^{\alpha} (p_2(x))^{1-\alpha}   \dx.
$$

We shall revisit this technique using the wider scope of {\em generalized weighted means}.

By definition, a mean $M(a,b)$ is a {\em smooth} function such that $\min(a,b)\leq M(a,b) \leq \max(a,b)$.
Similarly, we can define a weighted mean as a smooth function $M(a,b;\alpha)$ that fulfills the interness property
$M(a,b;\alpha)\in[\min(a,b),\max(a,b)]\ \forall \alpha\in[0,1]$. 
Let us consider the {\em quasi-arithmetic means} (also called Kolmogorov-Nagumo $f$-means~\cite{mean-kolmogorov-1930,mean-nagumo-1930}) for a strictly monotonous generator function $f$:

\begin{lemma}[\cite{meanvalueaxiomatization-1948}]
The quasi-arithmetic weighted mean $M_f(a,b;\alpha)=f^{-1}(\alpha f(a)+(1-\alpha)f(b))$ of two real values $a$ and $b$ for a strictly monotonic function $f$ satisfies the interness property: $\min(a,b) \leq M_f(a,b;\alpha)\leq \max(a,b)$.
\end{lemma}

\begin{proof}
Assume $f$ is strictly increasing and $a\leq b$, then $f(a)\leq f(b)$ and $f(a)\leq \alpha f(a)+(1-\alpha)f(b)\leq f(b)$.
Thus $a\leq M_f(a,b;\alpha)\leq b$. If $b\leq a$, we similarly have $b\leq M_f(a,b;\alpha)\leq a$. Therefore
$\min(a,b) \leq M_f(a,b;\alpha)\leq \max(a,b)$. The proof is identical when $f$ is strictly decreasing.
\end{proof}

Interestingly, one important property of quasi-arithmetic means is their {\em dominance relationship}.
That is, if $f(x)\leq g(x)$ then $M_f(a,b;\alpha) \leq M_g(a,b;\alpha)$  (with equality when $a=b$). 
This property generalizes the well known arithmetic-geometric-harmonic (AGH) inequality property of Pythagorean means:

\begin{equation}
M_{f_a}(a,b;\alpha) \geq M_{f_g}(a,b;\alpha) \geq M_{f_h}(a,b;\alpha),
\end{equation}
with $f_a(x)=x$, $f_g(x)=\log x$ and $h(x)=1/x$ denoting the generators for the arithmetic, geometric, and harmonic means, respectively.
That is, we have:
\begin{equation}
\alpha a+(1-\alpha)b \geq  a^{\alpha}b^{1-\alpha} \geq \frac{ab}{\alpha a+(1-\alpha)b}.
\end{equation}

Similarly to Chernoff~\cite{Chernoff-1952}, we define the following {\em generalized affinity coefficient} $\rho_f$ using generalized weighted means as follows:

\begin{definition}\label{def:rhof}
The Chernoff-type similarity coefficient (affinity) for a strictly monotonous function $f$ is defined by:
\begin{equation}\label{eq:genrho}
\rho_*^f (p_1,p_2)= \min_{\alpha\in [0,1]} \int M_f(p_1(x), p_2(x); \alpha) \dx\leq \int p_1(x)\dx = 1,
\end{equation}
\end{definition}
and define the generalized Chernoff information as:
\begin{definition}
The Chernoff-type information for a strictly monotonous function $f$ is defined by:
\begin{equation}
C_f(p_1,p_2) =-\log \rho_*^f(p_1,p_2)= \max_{\alpha\in [0,1]} -\log  \int M_f(p_1(x), p_2(x); \alpha) \dx \geq 0.
\end{equation}
\end{definition}

\begin{corollary}
The traditional Chernoff similarity, information, and upper bound are obtained by choosing the weighted geometric mean, by setting the generator $f_{\mathrm{Chernoff}}(x)=\log(x)$ (with $f^{-1}_{\mathrm{Chernoff}}(x)=\exp(x)$).
We get $M_{f_{\mathrm{Chernoff}}}(p_1(x),  p_2(x); \alpha)= p_1(x)^\alpha p_2(x)^{1-\alpha}$.
\end{corollary}

When we do not optimize over the parameter $\alpha$, but assume it fixed to $\frac{1}{2}$, we extend the Bhattacharyya coefficient and divergence as follows:
\begin{definition}\label{def:genbhat}
The generalized skew Bhattacharyya-type similarity coefficient (affinity) for a strictly monotonous function $f$ is defined by:
\begin{equation}\label{eq:genrhoalpha}
\rho^f_{\alpha} (p_1,p_2)=  \int M_f(p_1(x), p_2(x); \alpha) \dx\leq \int p_1(x)\dx = 1,
\end{equation}
and the generalized skew Bhattacharyya-type divergence is defined as $B^f_{\alpha}=-\log \rho^f_{\alpha} (p_1,p_2)$.
The  generalized Bhattacharyya coefficient $\rho^f(p_1,p_2) =\int M_f(p_1(x), p_2(x); \frac{1}{2}) \dx$ and divergence  $B^f(p_1,p_2)=-\log\rho^f(p_1,p_2)$.
\end{definition}

\begin{theorem}
Using quasi-arithmetic means, we  can bound the probability of error as follows:
\begin{equation}\label{eq:genpe}
P_e = \int \min(w_1 p_1(x),w_2 p_2(x))\dx \leq \int M_f(w_1p_1(x), w_2p_2(x); \alpha) \dx. 
\end{equation}
The upper bound proves useful for well-chosen $f$ yielding closed-form expression of the rhs.
\end{theorem}

In particular, by choosing the power means $M_{f_\beta}$ obtained for $f_\beta(x)=x^\beta$, we get a {\em tight bound} in the limit case since
 $M_{f_\beta}(p,q)\rightarrow \min(p,q)$ when $\beta\rightarrow-\infty$.
However, in order for the generalized affinity, distance and upper bound to be useful, we need to be able to compute them in {\em closed-form} for some statistical distribution families. 
We illustrate how to derive closed form formula for $\rho_f$ and closed form upper bounds on $P_e$ for several statistical distribution families.

The generalized Bhattacharyya upper bound $\rho_{\frac{1}{2}}^f$ is obtained by setting $\alpha=\frac{1}{2}$:
\begin{eqnarray} 
P_e &=& \int \min(w_1 p_1(x),w_2 p_2(x))\dx \leq  B(w_1 p_1(x),w_2 p_2(x)),\\
B_f(w_1 p_1(x),w_2 p_2(x)) &=& \int M_f\left(w_1p_1(x), w_2p_2(x); \frac{1}{2}\right) \dx.  
\end{eqnarray}

In order for the Chernoff information to improve over the Bhattacharyya bound, we need the quasi-arithmetic $\alpha$-weighted mean to be a {\em convex function} with respect to parameter $\alpha$.
For the geometric mean, we check that:

\begin{equation}
M_{f_g}(a,b;\alpha) = e^{\alpha \log\frac{a}{b} + \log b},
\end{equation}
is strictly convex with respect to $\alpha$ (since $\frac{\mathrm{d}^2}{\mathrm{d}\alpha^2}M_{f_g}(a,b;\alpha) = (\log\frac{a}{b})^2 M_{f_g}(a,b;\alpha)>0$ for $a\not =b$).
Similarly,  the weighted harmonic mean is convex with respect to $\alpha$ since $\frac{\mathrm{d}^2}{\mathrm{d}\alpha^2}M_{f_h}(a,b;\alpha) =2ab(a-b)^2(\alpha(a-b)+b)^{-3}>0$.

\begin{remark}
Note that not all quasi-arithmetic means yield convex weighted means.
Indeed, let $M_f'(a,b;\alpha)= f(\alpha(f^{-1}(a)-f^{-1}(b))+f^{-1}(b))$ 
then $\frac{\mathrm{d}^2}{\mathrm{d}\alpha^2}M_f'(a,b;\alpha)= (f^{-1}(a)-f^{-1}(b))^2 f''(\alpha(f^{-1}(a)-f^{-1}(b))+f^{-1}(b))$.
Functions $f$ and $f^{-1}$ are strictly monotonous but can be convex, concave, or arbitrary in general.
\end{remark}

\begin{remark}
Sometimes, we prefer to parameterize the weighted mean as the smooth interpolant from $a$ to $b$, when $\alpha$ varies from $0$ to $1$ (kind of geodesic parameterization).
In that case, we may prefer the parameterization $M_f(a,b;\alpha')=f^{-1}((1-\alpha') f(a)+\alpha' f(b))$.
This is not important for Bhattacharyya-type symmetric bounds nor for Chernoff-type bounds that optimize over the $\alpha$ range (or equivalently over the $\alpha'$ range).
\end{remark}


Let us now examine how the ``quasi-arithmetic bounding techniques'' apply for several families of statistical distributions.

\section{Some illustrating examples\label{sec:apps}}

\subsection{Geometric means and the Chernoff bound for exponential families}
First, we recall the well-known formula~\cite{Kailath-1967,2011-brbhat} for the case of exponential families.

In order to compute a closed form for the right-hand side of Eq.~\ref{eq:rhs}:

\begin{eqnarray}
P_e &\leq&  \int M_f(w_1p_1(x), w_2p_2(x); \alpha) \dx,\label{eq:rhs}\\
&\leq& \int f^{-1} ( \alpha  f(w_1p_1)+ (1-\alpha) f(w_2p_2) ) \dx,
\end{eqnarray}
we consider the {\em geometric mean} obtained for $f(x)=\log x$. Since $p_1(x)=\exp (x^\top \theta_1-F(\theta_1))$ and
$p_2(x)=\exp (x^\top \theta_2-F(\theta_2))$ belong to the  exponential families, we get:

\begin{eqnarray}
M_f(w_1p_1(x), w_2p_2(x); \alpha) &=& e^{\alpha\log w_1p_1(x) + (1-\alpha)\log w_2p_2(x)} ,\\
&=& w_1^{\alpha}w_2^{1-\alpha}p_1^\alpha(x)p_2^{1-\alpha}(x).
\end{eqnarray}
It follows that:

\begin{equation}
P_e \leq w_1^\alpha w_2^{1-\alpha} \int f^{-1} ( \alpha  f(p_1(x))+ (1-\alpha) f(p_2(x)) ) \dx.
\end{equation}

\begin{remark}
In fact, the geometric mean is a limit case of a family of linear-scale free means defined for $f_\alpha(x)=x^{\frac{1-\alpha}{2}}$ with $f_1(x)=\log x$. In general, in order to slide the {\em a priori} weights $w_1$ and $w_2$ out of the integral, we would like to use a homogeneous function $f$ (with $f(\lambda x)=g(\lambda)f(x)$).
\end{remark}

Furthermore, since $m_\alpha(x;\theta_1,\theta_2)= \alpha  \log(p_1(x))+ (1-\alpha) \log(p_2(x))= x^\top (\alpha \theta_1+(1-\alpha)\theta_2)-\alpha F(\theta_1)-(1-\alpha)F(\theta_2)$, we would like to get $f^{-1}(m_\alpha(x;\theta_1,\theta_2))$ as $c_{\theta_1,\theta_2;\alpha}p(x;\theta_{12}^{(\alpha)})$ so that we can slide the integral operand inside the expression, and use the fact that we recognize a member $\theta_{12}^{(\alpha)}$ of the exponential family so that its integration over the support is $1$.
For members of the {\em same} exponential family, we have:

\begin{eqnarray}
f^{-1}(m_\alpha(x;\theta_1,\theta_2)) &=& e^{F(\alpha\theta_1+(1-\alpha)\theta_2)-\alpha F(\theta_1)-(1-\alpha)F(\theta_2) } p(x;\alpha\theta_1+(1-\alpha)\theta_2),\\
&=& e^{-J_F^{(\alpha)}(\theta_1,\theta_2)} p(x;\underbrace{\alpha\theta_1+(1-\alpha)\theta_2}_{\theta_{12}^{(\alpha)}})
\end{eqnarray}

Thus

\begin{equation}
P_e \leq w_1^\alpha w_2^{1-\alpha} e^{-J_F^{(\alpha)}(\theta_1,\theta_2)} \int p(x;\alpha\theta_1+(1-\alpha)\theta_2) \dx.
\end{equation}
Since the natural parameter space $\Theta$ is convex for exponential families, we have $\theta_{12}^{(\alpha)} = \alpha\theta_1+(1-\alpha)\theta_2\in\Theta$ and therefore
 $\int p(x;\alpha\theta_1+(1-\alpha)\theta_2) \dx=1$.
We end up with:
 
\begin{equation}
P_e \leq  \min_{\alpha\in[0,1]} w_1^\alpha w_2^{1-\alpha} e^{-J_F^{(\alpha)}(\theta_1,\theta_2)}
\end{equation}

\begin{definition}
The $\alpha$-Chernoff distance is an asymmetric statistical distance defined by $\rho_\alpha(p_1,p_2)=-\log p_1^{\alpha}(x)p_2^{1-\alpha}(x)\dx \geq 0$.
The Bhattacharrya symmetric distance $\rho_{\frac{1}{2}}(p_1,p_2)$ is a particular member of the family of $\alpha$-Chernoff distances.
\end{definition}
Thus we always have $\rho_{\frac{1}{2}}(p_1,p_2) \geq \rho_*(p_1,p_2)=\min_{\alpha\in[0,1]} \rho_\alpha(p_1,p_2)$.

The optimal Chernoff bound is obtained  for the optimized weight $\alpha^*$ that has been characterized geometrically on the statistical manifold~\cite{ChernoffGeometry-2013}.

To derive other Chernoff-type upper bounds, we shall therefore consider non-exponential families of distributions. 
The most prominent family, to start with, is the Cauchy family (a member of the Student $t$-distribution families) and the multivariate Pearson type VII elliptical distributions~\cite{PearsonTypeVII-2010} (with its scaled multivariate $t$-distributions).

\subsection{Harmonic means and the Chernoff-type bound for Cauchy distributions\label{sec:Cauchy}}

Consider the family of Cauchy distributions  with density:

\begin{equation}
p(x;s) = \frac{1}{\pi} \frac{s}{x^2+s^2},
\end{equation}
defined over the support $\mathbb{R}$. 
This is a {\em scale family} with heavy tails indexed by a scale parameter $s$: 
\begin{equation}
p(x;s)=\frac{1}{s}p_0\left(\frac{x}{s}\right), \qquad p_0(x)=\frac{1}{\pi} \frac{1}{x^2+1},
\end{equation}
where $p_0(x)$ denotes the standard Cauchy distribution $C_0$.
Cauchy distributions do not belong to the exponential families.
(Indeed, the mean is undefined.)
Let us take the harmonic mean $M_H=M_f$ defined for the strictly monotonous generator $f(x)=f^{-1}(x)=\frac{1}{x}$.
We have:

\begin{eqnarray}
P_e = \int \min(w_1 p_1(x),w_2 p_2(x)) \dx \leq M_f(w_1 p_1(x),w_2 p_2(x); \alpha) \dx.
\end{eqnarray}

Wlog., to simplify calculations exhibiting the method, consider $w_1=w_2=\frac{1}{2}$.

\begin{eqnarray}
P_e &\leq& \int M_H(\frac{1}{2} p_1(x),\frac{1}{2} p_2(x); \alpha) \dx,\\
&\leq&  \frac{1}{2} \int \frac{p_1(x)p_2(x)}{(1-\alpha)p_1(x)+\alpha p_2(x)} \dx,\\
&\leq& \frac{1}{2}  \int \frac { \frac{s_1}{\pi(x^2+s_1^2)}  \frac{s_2}{\pi(x^2+s_2^2)} }{
(1-\alpha)\frac{s_1}{\pi(x^2+s_1^2)} + \alpha \frac{s_2}{\pi(x^2+s_2^2)}   } \dx,\\
&\leq& \frac{1}{2} \int \frac{s_1s_2}{\pi ((1-\alpha) s_1 (x^2 + s_2^2) + \alpha  s_2 (x^2 + s_1^2)   )} \dx,\\
&\leq& \frac{1}{2} \int \frac{s_1s_2}{\pi  ( ((1-\alpha) s_1+\alpha s_2)x^2 + (1-\alpha) s_1s_2^2+\alpha  s_2 s_1^2 )} \dx,\\
&\leq&\frac{1}{2} \frac{s_1s_2}{((1-\alpha) s_1+\alpha s_2) s_{\alpha}} \underbrace{\int \frac{1}{\pi} \frac{s_{\alpha}}{x^2+ s_{\alpha}^2} \dx}_{=1},\label{eq:cauchyfactor}
\end{eqnarray}
since $s_{\alpha}>0$ belongs to the parameter space $\Theta$, 
with:
\begin{equation}
s_{\alpha} =  \sqrt { \frac{(1-\alpha) s_1s_2^2+\alpha s_2s_1^2}{(1-\alpha) s_1+\alpha s_2}}.
\end{equation}

Thus the weighted  harmonic mean provides a Chernoff-type upper bound: 
\begin{equation}
P_e \leq  \frac{1}{2} \frac{s_1s_2}{((1-\alpha) s_1+\alpha s_2) \sqrt { \frac{(1-\alpha) s_1s_2^2+\alpha s_2s_1^2}{(1-\alpha) s_1+\alpha s_2}}}.
\end{equation} 
A sanity check $s_1=s_2=s$ shows that $P_e=\frac{1}{2}$, as expected (class-conditional distributions are not distinguishable).

The Bhattacharrya-type bound obtained for $\alpha=\frac{1}{2}$ yields:

\begin{equation}
s_{\frac{1}{2}} =  \sqrt{\frac{s_1s_2^2+s_2s_1^2}{s_1+s_2}}, 
\end{equation}
and the upper bound:
\begin{equation}
P_e\leq \frac{s_1s_2}{\sqrt{(s_1+s_2)(s_1s_2^2+s_2s_1^2)}}.
\end{equation}

The harmonic mean $M_H(a,b;\alpha)$ is {\em linear-scale free}: 
$M_H(\lambda a,\lambda b;\alpha)=\lambda M_H(a,b;\alpha)$.
Let $\lambda=\frac{s_2}{s_1}$. Then we write the probability of error as:

\begin{equation}\label{eq:cauchyalpha}
P_e \leq \frac{1}{2} \frac{\lambda}{\sqrt{(1-\alpha+\alpha\lambda)((1-\alpha)\lambda^2+\alpha\lambda)}} =P_e^{(\alpha)}.
\end{equation}

In particular, we upper bound the probability of error by the following Bhattacharyya bound:
\begin{equation}
P_e \leq P_e^{(\frac{1}{2})} =\frac{\sqrt{\lambda}}{\lambda+1}.
\end{equation}

The Chernoff-type bound proceeds by minimizing Eq.~\ref{eq:cauchyalpha} with respect to $\alpha$.
We have $P_e^{(0)}=P_e^{(1)}=\frac{1}{2}$. To find the minimum value over the $\alpha$-range $[0,1]$, we study function $P_e^{(\alpha)}$.
Using a computer-algebra system\footnote{Namely, Wolfram Alpha online, \url{http://www.wolframalpha.com}}, we find that it is a convex function that always admits a minimum at $\alpha=\frac{1}{2}$. 
Indeed, the derivative of $P_e$ with respect to $\alpha$ is:

\begin{equation}
\frac{\mathrm{d}}{\mathrm{d}\alpha} P_e(\alpha;\lambda)= 
\frac{\frac{1}{4} (\lambda-1)^2 \lambda^2 (2\alpha-1) }{(\lambda ((\lambda-1)\alpha+1) (-\alpha\lambda+\lambda+\alpha))^{3/2}},
\end{equation} 
that is zero if and only if $\alpha=\frac{1}{2}$.
This is a remarkable example that shows that the Chernoff  bound  amounts to the Bhatthacharrya bound.

Let us compute the total variation distance between two scaled Cauchy distributions $a_1 p(x;s_1)$ and $a_2 p(x;s_2)$
defined over the real-line support $\mathbb{R}$ with:
\begin{equation}
p(x;s) = \frac{1}{\pi} \frac{s}{x^2+s^2}
\end{equation}
For $s_1\not = s_2$, the two distinct  positive densities intersect in exactly two values of the support:

\begin{eqnarray}
x_{1} &=& - \frac{\sqrt{s_1s_2(a_2 s_1- a_1s_2)}}{\sqrt{a_1s_1-a_2s_2}} , \\ 
x_{2} &=&  \frac{\sqrt{s_1s_2(a_2s_1-a_1s_2)}}{\sqrt{a_1s_1-a_2s_2}} 
\end{eqnarray}
In particular, when $a_1=a_2=1$ (ie., probability densities), we have $x_1=-\sqrt{s_1s_2}$ and $x_2=\sqrt{s_1s_2}$.
By abuse of notations, we let $x_0=-\infty$ and $x_3=\infty$, and apply the generic 1D total variation formula of Eq.~\ref{eq:tv1d} with $k=2$:

\begin{equation}
\TV(a_1 p_1,a_2 p_2) = \frac{1}{2} \sum_{i=1}^{k+1} \left|   (P_1(x_i)+P_2(x_{i+1})-P_1(x_{i-1})-P_2(x_i))\right|. 
\end{equation}

Since the Cauchy scaled cumulative distribution is $P_i(x)=a_i (\frac{1}{\pi} \arctan(\frac{x}{s_i}) + \frac{1}{2})$ with $P_i(x_0)=0$ and  $P_i(x_3)=a_i$.

In particular, the probability of error when $w_1=w_2=\frac{1}{2}$ is:
\begin{eqnarray}
P_e &=& \frac{1}{2}(1-\TV(p_1,p_2)),\\
\TV(p_1,p_2) &=& |P_1(\sqrt{s_1s_2})-P_1(-\sqrt{s_1s_2})-P_2(\sqrt{s_1s_2})+P_2(-\sqrt{s_1s_2})|.
\end{eqnarray}
That is, we find an exact analytic expression for the total variation (and hence for Bayes error $B_e$ and the probability of error $P_e$):
\begin{eqnarray}
\lefteqn{\TV(p_1,p_2) = \frac{1}{\pi} \left ( \arctan\left(\sqrt{\frac{s_2}{s_1}}\right)-\arctan\left( -\sqrt{\frac{s_2}{s_1}}\right)\right.}\nonumber\\
&& 
\left. +\arctan \left(-\sqrt{\frac{s_1}{s_2}}\right) -\arctan \left(\sqrt{\frac{s_1}{s_2}}\right) \right).
\end{eqnarray}

Using the identity $\arctan(x)+\arctan(1/x)=\frac{\pi}{2}$ and the fact that $\arctan(-x)=-\arctan(x)$, we get a more compact formula:

\begin{equation}
\TV(p_1,p_2) = \frac{2}{\pi}\left( \arctan\left(\sqrt{\frac{s_2}{s_1}}\right)-\arctan\left(\sqrt{\frac{s_1}{s_2}}\right)  \right).
\end{equation}

\begin{remark}
Note that for Cauchy distributions with scale parameter $s_1$ and $s_2$, we have $\TV(s_1,s_2)=\TV(\lambda s_1,\lambda s_2), \forall \lambda>0$ (because $\sqrt{\frac{\lambda s_1}{\lambda s_2}} = \sqrt{\frac{s_1}{s_2}}$). Therefore, we may renormalize by considering $s_1\leftarrow 1$ and $s_2\leftarrow \frac{s_2}{s_1}$.
\end{remark}

It follows that the probability of error is:

\begin{eqnarray}
P_e &=& \frac{1}{2}-\frac{1}{\pi} \left(\arctan(\sqrt{\lambda})-\arctan(\sqrt{1/\lambda}) \right),\\
&=&1-\frac{2}{\pi}\arctan(\sqrt{\lambda}),\quad \lambda=\frac{s_2}{s_1}.
\end{eqnarray}

Consider the following numerical example: $s_1=10$ and $s_2=50$ ($w_1=w_2=\frac{1}{2}$). 
Then,  we get $P_e\sim 0.2677$, the Bhattacharyya-type bound $B \sim 0.3726$ and the Chernoff-type bound $C=B$.

\begin{remark}
Let us study the tightness of the   upper bound.  We can express analytically the gap as
$\Delta = C-P_e = C-\frac{1}{2}+\TV(p_1,p_2)$. That is, we get:

\begin{equation}
\Delta(\lambda)=\frac{\sqrt{\lambda}}{1+\lambda}-1+\frac{2}{\pi}\arctan(\sqrt{\lambda})>0.
\end{equation}

Note that $\lambda=1$, we have $\Delta(1)=0$: That is, the gap is tight when distributions coincide.
Using a computer-algebra system, we find that the gap is maximized at $\lambda=\frac{2+\pi}{\pi-2}\sim 4.5$ and $\Delta_{\max}\sim 0.1$.
\end{remark}

This Cauchy case study illustrates well that the geometric mean may not always be the most appropriate weighted mean to use to upper bound $P_e$.
Note that for Cauchy distributions, we also obtained an analytic form of $P_e$ and $B_e$ using the total variation distance expressed using the cumulative distribution.

We showed that the harmonic mean is tailored to derive closed-form bound for the Cauchy distributions.
However, we may apply the harmonic mean to other distributions. This has in fact be done in the literature detailed in the following remark:
\begin{remark}
The harmonic bean $M_H(a,b)=\frac{2ab}{(a+b)}$ has been used for defining the {\em nearest neighbor error bound}~\cite{Fukunaga-1990}, always better than the Bhattacharrya bound (obtained for $\alpha=\frac{1}{2}$). 
\end{remark}
(Note that for exponential families, the nearest neighbor error bound is not available in closed-form.)

\subsection{Pearson type VII distributions\label{sec:Pearson}}

Consider the $d$-dimensional elliptically symmetric Pearson type VII distribution~\cite{PearsonTypeVII-2010,BayesEllipticalDist-2012})   with density:

\begin{equation}
p(x;\mu,\Sigma,\lambda) = \pi^{-\frac{d}{2}} \frac{\Gamma(\lambda)}{\Gamma(\lambda-\frac{d}{2})} |\Sigma|^{-\frac{1}{2}} (1+(x-\mu)^\top \Sigma^{-1} (x-\mu))^{-\lambda}
\end{equation}
where $\Gamma(z)=\int_0^\infty t^{z-1} e^{-t}\dt$ denotes the Gamma function extending the factorial function (i.e., $\Gamma(1)=1$, $\Gamma(x+1)=x\Gamma(x)$ and $\Gamma(n)=(n-1)!$ for $n\in\mathbb{N}$). Parameter $\lambda>\frac{d}{2}$ represents the degree of freedom~\cite{PearsonTypeVII-2010}.

For sake of simplicity, wlog., let us consider the zero-centered sub-family with $\mu=0$, $v=1$ and $\lambda>\frac{d}{2}$ fixed.
This sub-family is defined on the cone of symmetric positive definite matrices $\Theta=\{\Sigma\ |\ \Sigma\succ 0\}$ with density:
\begin{equation}
p(x;\Sigma)=  c_d(\lambda) |\Sigma|^{-\frac{1}{2}} (1+x^\top \Sigma^{-1}x)^{-\lambda},
\end{equation}
where $c_d(\lambda)=\pi^{-\frac{d}{2}} \frac{\Gamma(\lambda)}{\Gamma(\lambda-\frac{d}{2})}$ is the normalizing constant~\cite{NormalizingConstantDistribution-2012}.
This multivariate family does not belong to the exponential families.
Consider the $\alpha$-weighted $f$-mean with $f(x)=x^{-\frac{1}{\lambda}}$, for prescribed $\lambda>\frac{d}{2}$ (and $f^{-1}(x)=x^{-\lambda}$).

We have:

\begin{eqnarray}
\alpha f(p_1) &=& \alpha c_d(\lambda)^{-1/\lambda} |\Sigma_1|^{\frac{1}{2\lambda}} (1+x^\top \Sigma_1^{-1}x),\\
(1-\alpha) f(p_2)&=& (1-\alpha) c_d(\lambda)^{-1/\lambda} |\Sigma_2|^{\frac{1}{2\lambda}} (1+x^\top \Sigma_1^{-2}x).
\end{eqnarray}
Let $c_1=c_d(\lambda)^{-1/\lambda} |\Sigma_1|^{\frac{1}{2\lambda}} $ and $c_2=c_d(\lambda)^{-1/\lambda} |\Sigma_2|^{\frac{1}{2\lambda}} $.
Denote by:

\begin{eqnarray}
c_\alpha &=&\alpha c_1+(1-\alpha)c_2,\\
&=& c_d(\lambda)^{-1/\lambda} (\alpha  |\Sigma_1|^{\frac{1}{2\lambda}}  + (1-\alpha) |\Sigma_2|^{\frac{1}{2\lambda}}),
\end{eqnarray}
we get:

\begin{eqnarray}
\alpha f(p_1) +(1-\alpha) f(p_2) &=& \alpha c_1(1+x^\top \Sigma_1^{-1} x) + (1-\alpha)c_2(1+x^\top \Sigma_2^{-1} x),\\
&=& c_\alpha (1 + x^\top \Sigma_\alpha^{-1} x),
\end{eqnarray}
with

\begin{eqnarray}
\Sigma_\alpha^{-1} &=& \frac{\alpha c_1\Sigma_1^{-1}+(1-\alpha)c_2\Sigma_2^{-1}}{c_\alpha},\\
&=& \frac{\alpha |\Sigma_1|^{\frac{1}{2\lambda}} \Sigma_1^{-1}+(1-\alpha)|\Sigma_2|^{\frac{1}{2\lambda}} \Sigma_2^{-1}}{  (\alpha |\Sigma_1|^{\frac{1}{2\lambda}} + (1-\alpha)|\Sigma_2|^{\frac{1}{2\lambda}})} \succ 0.
\end{eqnarray}

Therefore,

\begin{eqnarray}
f^{-1}(\alpha f(p_1) +(1-\alpha) f(p_2))&=&  c_\alpha^{-\lambda} (1 + x^\top \Sigma_\alpha^{-1} x)^{-\lambda},\\
&=& c_\alpha^{-\lambda} |\Sigma_\alpha|^{\frac{1}{2}} \frac{1}{c_d(\lambda)} p(x;\Sigma_\alpha)
\end{eqnarray}

It follows that:

\begin{eqnarray}
P_e &\leq&  \frac{1}{2} (\alpha  |\Sigma_1|^{\frac{1}{2\lambda}}  + (1-\alpha) |\Sigma_2|^{\frac{1}{2\lambda}})^{-\lambda} |\Sigma_\alpha|^{\frac{1}{2}} \underbrace{\int p(x;\Sigma_\alpha)\dx}_{=1},\\
&=& \frac{1}{2} (\alpha  |\Sigma_1|^{\frac{1}{2\lambda}}  + (1-\alpha) |\Sigma_2|^{\frac{1}{2\lambda}})^{-\lambda} |\Sigma_\alpha|^{\frac{1}{2}}.
\end{eqnarray}
since $\Sigma_\alpha\in\Theta$.

The Pearson type VI distribution is related to the multivariate $t$-distributions~\cite{NormalizingConstantDistribution-2012}.

\subsection{Central multivariate $t$-distributions\label{sec:mvt}}
The  multivariate $t$-distribution (MVT, centered at $\mu=0$) with $\nu\geq 1$ {\em degrees of freedom} is defined for a positive definite matrix $\Sigma\succ 0$ (the {\em scale matrix}) by the following density:

\begin{equation}
p(x;\Sigma) = c_{d,\nu} |\Sigma|^{-\frac{1}{2}} \left(1+\frac{1}{\nu} x^\top \Sigma^{-1} x \right)^{-\frac{\nu+d}{2}},
\end{equation}
where $c_{d,\nu}=\frac{\Gamma(\frac{\nu+d}{2})}{\Gamma(\frac{\nu}{2})(\nu\pi)^{\frac{d}{2}}}$ is the  constant  normalizing the distribution.
The covariance matrix  is $\frac{\nu}{\nu-2}\Sigma$.

Let $t=-\frac{\nu+d}{2}$, and consider $f(x)=x^{\frac{1}{t}}$, with functional inverse $f^{-1}(x)=x^t$.
Using a technique similar to the Pearson bound, after massaging the mathematics, we find that (for $w_1=w_2=\frac{1}{2}$) 
$P_e \leq \frac{1}{2}\rho^{\mathrm{MVT}}_\alpha(\Sigma_1,\Sigma_2)$ (for $\alpha\in [0,1]$),
with:

\begin{equation}\label{eq:mvtalpha}
\rho^{\mathrm{MVT}}_{\alpha}(\Sigma_1,\Sigma_2) = (\alpha |\Sigma_1|^{-\frac{1}{2t}} + (1-\alpha) |\Sigma_2|^{-\frac{1}{2t}}  )^t  |\Sigma_{\alpha}'|^{\frac{1}{2}},
\end{equation}
and

\begin{equation}
\Sigma_{\alpha}' = \left( \frac{\alpha |\Sigma_1|^{-\frac{1}{2t}}\Sigma_1^{-1} + (1-\alpha) |\Sigma_2|^{-\frac{1}{2t}} \Sigma_2^{-1} }{\alpha |\Sigma_1|^{-\frac{1}{2t}} + (1-\alpha) |\Sigma_2|^{-\frac{1}{2t}} }  \right)^{-1}.
\end{equation}

Note that when $\nu\rightarrow\infty$, and $t=-\frac{2}{\nu+d}\rightarrow 0$, the multivariate $t$-distribution (MVT) tend to a multivariate Normal distribution (MVN) with covariance matrix $\Sigma$. The power mean induced by  $f(x)=x^{\frac{1}{t}}$ (with $t=-\frac{\nu+d}{2}$) tends to the geometric mean, and we get  
 the well-known Bhattacharyya coefficient bound (for $\alpha=\frac{1}{2}$) on central multivariate Gaussians~\cite{2011-brbhat} (see Eq.~\ref{eq:rhomvn}):
 
\begin{equation}
P_e \leq \frac{1}{2} \rho^{\mathrm{MVN}}(\Sigma_1,\Sigma_2),\quad
\rho^{\mathrm{MVN}}(\Sigma_1,\Sigma_2) = \frac{|\Sigma_1|^{\frac{1}{4}}|\Sigma_2|^{\frac{1}{4} }}{|\frac{1}{2}\Sigma_1 + \frac{1}{2}\Sigma_2|^{\frac{1}{2}}}.
\end{equation}

\subsection{Assessing the Bhattacharyya-type and the Chernoff-type upper bounds\label{sec:tightness}}

Let us recall the inequality on the probability of error $P_e$ between two distributions $p_1$ and $p_2$ with equal prior ($w_1=w_2=\frac{1}{2}$):

\begin{equation}
P_e(p_1,p_2) = \frac{1}{2}\left(1-\TV(p_1,p_2) \right)  \leq   \frac{1}{2} \rho_*^{f}(p_1,p_2) \leq  \frac{1}{2} \rho^{f}(p_1,p_2) \leq \frac{1}{2}.
\end{equation}

The left hand side has been elucidated in Eq.~\ref{eq:petv} and the right hand side is the Chernoff-type/Bhattacharyya-type similarity coefficients which should be available in closed-form for fast calculation.
We are interested in characterizing the gaps $\Delta=\rho(p_1,p_2)-P_e$ and $\Delta_*=\rho_*(p_1,p_2)-P_e$ between the Bhattacharyya/Chernoff upper bounds  and $P_e$  (with $\Delta_*\leq \Delta$). 
We start with some simple cases of univariate distributions, where $P_e$ can be expressed analytically, and then considered the multivariate distributions where $P_e$ need to be stochastically estimated.

\subsubsection{Simple cases: Univariate distributions}

For univariate densities,  we may calculate the total variation by computing the roots of $p_1(x)=p_2(x)$ and then using the cumulative distributions $P(t)=\int p(x\leq t) \dx$ to explicit a formula.
Assume $x_1, ..., x_k$ are the $k$ roots, and by abuse of notations, let $x_0=x_{\min}$ and $x_{k+1}=x_{\max}$ denote the extra endpoints of the distribution support ($x\in[x_{\min},x_{\max}]$). When the support is the full real line ($\supp(p_i)=\mathbb{R}$), we set $x_{\min}=-\infty$ and $x_{\max}=\infty$.
We have:

\begin{eqnarray}
\TV(p_1,p_2) &=& \frac{1}{2}   \sum_{i=1}^{k+1}  \left| \int_{x_{i-1}}^{x_{i}} (p_1(x)-p_2(x)) \dx\right|,\\
&=& \frac{1}{2} \left|\sum_{i=1}^{k+1}    (P_1(x_i)+P_2(x_{i+1})-P_1(x_{i-1})-P_2(x_i))\right|. \label{eq:tv1d}
\end{eqnarray}

This scheme allows one to compute the total variation distance of many {\em univariate} distributions like the Gaussian (see Eq.~\ref{eq:unitv}), Rayleigh, Cauchy, etc distributions. Therefore, we get analytic expressions of Bayes error relying on the {\em cumulative distribution function} (CDF) for many univariate distributions.

Those the Bhattacharyya gaps for the Cauchy or 1D Gaussians can be written analytically.
Section~\ref{sec:Cauchy} already addressed the gap for the Cauchy distributions.
Next, we consider the general multivariate case (that includes those univariate examples) and report on our numerical experiments.

\subsubsection{The case of multivariate distributions}
We consider multivariate $t$-distributions~\cite{mvtnorm-2013,mvt-2009} (MVT) that includes  
the multivariate normal (MVN) distributions in the limit case (when the number of degree of freedom tends to infinity).
To perform experiments, we used the {\tt mvtnorm} package\footnote{\url{http://cran.r-project.org/web/packages/mvtnorm/index.html}. To install the package, we used the command line: {\tt install.packages('mvtnorm\_0.9-9996.zip', repos = NULL)}} on the R software platform.\footnote{R can be freely downloaded at \url{http://www.r-project.org/}}

Since  the total variation (TV) does not admit a closed-form formula (nor the probability of error $P_e$), we estimate those quantities by 
performing stochastic integrations as follows:

\begin{eqnarray}
\widehat{P_e}(p_1,p_2)&=&\frac{1}{2}(1-\widehat{\TV}(p_1,p_2)),\\
  \widehat{\TV}(p_1,p_2) &=& \frac{1}{2n} \sum_{i=1}^n \frac{1}{p_1(x)}\left|p_1(x)-p_2(x)\right|,\\
  &=& \frac{1}{2n} \sum_{i=1}^n \left|1-\frac{p_2(x_i)}{p_1(x_i)}\right|,
\end{eqnarray}
where $x_1, ..., x_n$ are $n$ identically and independently variates of $p_1$.
Stochastic integration guarantees convergence to the true value in the limit: $\lim_{n\rightarrow\infty} \widehat{P_e}(p_1,p_2)=P_e(p_1,p_2)$.

To give a numerical example, consider central bidimensional $t$-distributions with $\nu=6$ and scale matrices $\Sigma_1=I$ and $\Sigma_2=10I$, where $I$ denotes the identity matrix. Running the {\tt  mvtTotalVariation(df,sigma1,sigma2,n)} code (see Appendix), we get
the following estimates for $\widehat{TV}$:  0.3210709 ($n=100$), 0.3479519 ($n=1000$), 0.3472926 ($n=10000$), 0.347538 ($n=100000$).
We chose $n=10000$ in the following experiments.

Consider central $t$-distributions ($\mu=0$).
We implemented the closed-form formula of Eq.~\ref{eq:mvtalpha} to calculate the $\alpha$-Chernoff coefficient
$\rho^{\mathrm{MVT}}_{\alpha}(\Sigma_1,\Sigma_2)$. The optimal Chernoff coefficient (and the exponent $\alpha^*$) is approximated by discretizing into 1000 steps the unit range for $\alpha$.
We consider $\Sigma_1=I$ and $\Sigma_2=\lambda I$ (with $\nu=6$) for $\lambda=d+1$ and various values of the dimension. 
(Indeed, after an appropriate ``whitening'' transformation~\cite{Fukunaga-1990}, we may assume wlog. that one parameter matrix is the identity while the other is diagonal.)
We report the experimental results  in Table~\ref{tab:mvt}.

 \begin{table}
\caption{Experimental results for central multivariate $t$-distributions (MVT) with $\nu=6$, $\Sigma_1=I$ and $\Sigma_2=\lambda I$.
We have: $\widehat{P_e} \leq \frac{1}{2}\rho_*^{\mathrm{MVT}}\leq \frac{1}{2}\rho^{\mathrm{MVT}}$ where $\rho_*^{\mathrm{MVT}}$ and  $\rho^{\mathrm{MVT}}$ denotes the Chernoff and Bhattacharyya coefficient, respectively. Observe that the Chernoff upperbound is tighter (but of the same order) than the  Bhattacharyya bound (for $\alpha^*\not =\frac{1}{2}$). Those upper bounds improve over the naive $\frac{1}{2}$ bound.
\label{tab:mvt}}
$$
\begin{array}{ll|llll}
\mathrm{dimension} & \lambda & \widehat{P_e} & \rho_{*}^{\mathrm{MVT}} &  \widehat{\alpha^*}& \rho^{\mathrm{MVT}} \\ \hline
d= 2  &\lambda= 3  & 0.3302  & 0.4471817 & 0.455  & 0.4475742 \\
d= 3 & \lambda= 4  & 0.2578  & 0.3951277 & 0.462  & 0.3956298 \\
d= 5 & \lambda= 6  & 0.16215  & 0.2943599 & 0.487  & 0.2944589\\ 
d= 10&  \lambda= 11  & 0.06045  & 0.1400655 & 0.548  & 0.141438 \\
d= 15 & \lambda= 16  & 0.02845  & 0.07442729 & 0.592  & 0.07841622 \\
d= 20 & \lambda= 21  & 0.0167  & 0.04396945  & 0.625  & 0.04945252 \\
\end{array}  
$$
\end{table}

\section{Conclusion\label{sec:concl}}
In this paper, we first reported a formula relating the Bayes error $B_e$ (including the probability of error $P_e$) to the total variation metric  $\TV$ defined on scaled distributions (see Theorem~1).
Second, we elucidated the Chernoff upper bound mechanism based on generalized weighted means:
Chernoff~\cite{Chernoff-1952} used the fact that $\min(a,b)\leq a^{\alpha}b^{1-\alpha}$ for $a,b>0$ and $\alpha\in[0,1]$ to derive an upper bound that turned out to be well-suited to the structure of exponential families~\cite{Kailath-1967}.
We interpreted the right-hand side of this   inequality as a  weighted geometric mean, and considered extending the upper bound construction using generalized weighted mean $M(a,b;\alpha)$. 
A mean $ M(a,b;\alpha)$ is indeed guaranteed to fall
 within its extrema by definition, thus yielding the bounds: $\min(a,b)\leq M(a,b;\alpha)\leq \max(a,b)$.
We considered the family of quasi-arithmetic means~\cite{meanvalueaxiomatization-1948,mean-kolmogorov-1930,mean-nagumo-1930} and showed how to derive new upper bounds by {\em coupling} the structure of the generalized mean with the structure of the probability distribution family at hand. 
We illustrated our method by considering three examples: The univariate Cauchy distributions, and the multivariate Pearson type VII and $t$-distributions (that includes the multivariate normal distributions in the limit case). For those families, we designed new affinity coefficients upper bounding $B_e$. The best value $\alpha^*$ of $\alpha$ yielding the tightest coefficient can be found by optimization on the statistical manifold~\cite{ChernoffGeometry-2013}.
We carried out numerical experiments that show that those novel upper bounds are helpful because not too distant to Bayes error (although not very tight), specially because they can be calculated in constant time using closed-form formula.
Otherwise, for more precise approximations,  the Bayes error can be estimated using computationally-intensive stochastic integrations.

Last but not least, this paper revealed novel interactions between the Bayes error  and statistical divergences: We show how to design
$P_e(p_1,p_2)\leq \rho^{f}_\alpha(p_1,p_2)$ upper bounds where  $\rho^{f}_\alpha$ is an affinity coefficient derived from a weighted quasi-arithmetic mean $M_f$ (skewed for Chernoff type and symmetric for a Bhattacharrya type).
Since we can transform any affinity coefficient $\rho^{f}_\alpha$ into a corresponding divergence by defining $D^{f}_\alpha=-\log \rho^{f}_\alpha$, we deduce that those novel statistical divergences $D^{f}_\alpha$ can be used to upper bound the probability of error: $P_e(p_1,p_2)\leq e^{-D^{f}_\alpha(p_1,p_2)}$.

\section*{Acknowledgments}
The author is grateful for the valuable comments of the Reviewers that led to this revised work.

\appendix

\section{Bayes error for class-conditional Gaussians\label{app:mvn}}

We summarize the formula for the Bayes error $B_e$ and the probability of error $P_e$ when dealing with univariate and multivariate normal  class-conditional distributions:
\begin{eqnarray}
B_e &=& \frac{a_1+a_2}{2}-\TV(a_1p_1,a_2p_2),\\
P_e &=& \frac{1}{2}-\TV(w_1p_1,w_2p_2),
\end{eqnarray}
where $a_1=w_1(c_{11}+c_{21})$ and $a_2=w_2(c_{12}+c_{22})$ are the positive weights derived from the cost design matrix and the {\em a priori} class weights.

First, consider the family of univariate normal distributions $\mathcal{F}=\{N(\mu,\sigma)\ |\ \mu\in\mathbb{R},\sigma\in\mathbb{R}_+\}$. 
Consider two distinct normal densities $p_1$ and $p_2$.
The two weighted densities $a_1 p_1(x) = a_2 p_2(x)$ intersect at exactly two positions when $\sigma_1\not =\sigma_2$ or in exactly one position, otherwise:
\begin{equation}
\frac{a_1}{\sigma_1\sqrt{2\pi}} e^{-\frac{1}{2}(\frac{x-\mu_1}{\sigma_1})^2} = 
\frac{a_2}{\sigma_2\sqrt{2\pi}} e^{-\frac{1}{2}(\frac{x-\mu_2}{\sigma_2})^2}. 
\end{equation}

Finding the roots amounts to solve the quadratic equation:

\begin{equation}
\left(\frac{x-\mu_1}{\sigma_1}\right)^2 - \left(\frac{x-\mu_2}{\sigma_2}\right)^2  -2\log \frac{a_1\sigma_2}{\sigma_1 a_2} = 0.
\end{equation}

When $\sigma_1=\sigma_2=\sigma$ (with $\mu_1\not= \mu_2$), we have one root:

\begin{equation}
x_1= \frac{\mu_1^2-\mu_2^2-2\log \frac{a_1\sigma_2}{\sigma_1 a_2}}{2(\mu_1-\mu_2)}.
\end{equation}

We find:

\begin{equation}
\TV(a_1p_1,a_2p_2) = \frac{1}{2} |a_2 \Phi(x_1;\mu_2,\sigma_2)- a_1 \Phi(x_1;\mu_1,\sigma_1)|,
\end{equation}
where $\Phi(x;\mu,\sigma)=\frac{1}{2}(1+\erf(\frac{x-\mu}{\sigma\sqrt{2}}))$ is the cumulative distribution, and $\erf(x)=\frac{1}{\sqrt{\pi}} \int_{-x}^x e^{-t^2} \dt$ denotes the error function.
That is,

\begin{eqnarray}
B_e &=& \frac{a_1+a_2}{2}-\frac{1}{2} \left|a_2 \erf(\frac{x_1-\mu_2}{\sigma\sqrt{2}}) - a_1 \erf (\frac{x_1-\mu_1}{\sigma\sqrt{2}})\right|,\\
x_1 &=& \frac{\mu_1^2-\mu_2^2-2\log \frac{a_1\sigma_2}{\sigma_1 a_2}}{2(\mu_1-\mu_2)}.
\end{eqnarray}

When $\sigma_1\not=\sigma_2$, the quadratic equation expands as $ax^2+bx+c=0$, and we have two distinct roots $x_1$ and $x_2$:

\begin{eqnarray}
a&=& \frac{1}{\sigma_1^2}-\frac{1}{\sigma_2^2},\\
b&=&  2\left(\frac{\mu_2}{\sigma_2}-\frac{\mu_1}{\sigma_1}\right)\\
c&=&  \left(\frac{\mu_1}{\sigma_1}\right)^2 - \left(\frac{\mu_2}{\sigma_2}\right)^2 -2\log \frac{a_1\sigma_2}{a_2\sigma_1}  \\
x_1&=&\frac{-b-\sqrt{\Delta}}{2a},\quad x_2=\frac{-b+\sqrt{\Delta}}{2a},
\end{eqnarray}
with $\Delta=b^2-4ac\geq 0$ and the total variation writes as follows:

\begin{eqnarray}\label{eq:unitv}
\lefteqn{\TV(a_1 p_1, a_2 p_2) =}\nonumber \\
&& \frac{1}{2} \left( 
\left|
\erf\left(\frac{x_1-\mu_1}{\sigma_1\sqrt{2}}\right)-\erf\left(\frac{x_1-\mu_2}{\sigma_2\sqrt{2}}\right)
\right|
+
\left|
\erf\left(\frac{x_2-\mu_1}{\sigma_1\sqrt{2}}\right)-\erf\left(\frac{x_2-\mu_2}{\sigma_2\sqrt{2}}\right)
\right|
\right)
\end{eqnarray}

Those formula generalize the probability of error reported in~\cite{BayesError1DNormal-2008}, p. 1375 to the most general case.

Second, consider the family of  multivariate normals distributions  
$\{ N(\mu,\Sigma)\ |\ \mu\in\mathbb{R}^d, \Sigma\succ 0\}$.
For densities $p_1$ and $p_2$ having the same covariance matrix $\Sigma$, the probability of error is reported in~\cite{BayesErrorNormal-2013}, even for degenerate covariance matrices $\Sigma$ by taking the pseudo-inverse matrix $\Sigma^{+}$:

\begin{equation}
P_e =  \frac{1}{2} - \frac{1}{2} \erf\left( \frac{1}{2\sqrt{2}} \| (\Sigma^+)^{\frac{1}{2}}(\mu_2-\mu_1) \|\right).
\end{equation}

When covariance matrices are distinct ($\Sigma_1\not =\Sigma_2$) but {\em linear classifiers} are considered, we also get  a closed-form formula~\cite{BayesErrorGaussian-2005} for the probability of error.
Otherwise, for the general case of {\em quadratic classifiers} of multivariate normals with distinct covariance matrices, no analytical formula is known. The best way to compute the probability of error is then by performing {\em 1D integration} of the {\em conditional density} of the {\em discriminant function}~\cite{Fukunaga-1990}.

Note that since two matrices can always be simultaneously diagonalized~\cite{Fukunaga-1990}, it is enough to consider the case of two Gaussians with the first covariance being set to the identity matrix $I$ and the second covariance matrix set to a diagonal matrix $\Lambda$.

\appendix

\section{R code}

\begin{lstlisting}
###
### Generalized Bhattacharyya and Chernoff upper bounds on Bayes error using quasi-arithmetic means
###
### (C) October 2013 Frank Nielsen (Frank.Nielsen@acm.org)

require(mvtnorm)

#
# Various functions
#

# Stochastic evaluation of the total variation metric distance
mvtTotalVariation <- function(df,sigma1,sigma2,n)
{
tv=0
dim=nrow(sigma1)
mu0=rep(0,dim)

x1=rmvt(n,sigma1,df)
x2=rmvt(n,sigma2,df)

for (i in 1:n) {

tv = tv+abs(1-(dmvt(x1[i,1:dim],mu0,sigma2,log=FALSE)/dmvt(x1[i,1:dim],mu0,sigma1,log=FALSE)))

tv = tv+abs(1-(dmvt(x2[i,1:dim],mu0,sigma1,log=FALSE)/dmvt(x2[i,1:dim],mu0,sigma2,log=FALSE)))
}

tv=0.5*tv/(2*n)
tv
}

# Evaluated from stochastic TV
mvtPe <- function(df,sigma1,sigma2,n)
{
0.5*(1-mvtTotalVariation(df,sigma1,sigma2,n))
}


#
# discretize alpha into steps (approximate alpha star)
#
mvtChernoffCoefficient <- function(df,sigma1,sigma2)
{
steps=10000
best=1.0; # worst tight coefficient

for (i in 1:steps)
{
alpha=(i/steps)

Cac=mvtAlphaChernoffCoefficient(alpha,df,sigma1,sigma2)

 

if (Cac<best)
{
alphastar=alpha
best=Cac
}

} # endfor

#cat("alphastar=",alphastar," best Chernoff coeff=",best)

c(best,alphastar)
}
 
mvtBhattacharryaCoefficient <- function(df,sigma1,sigma2)
{
mvtAlphaChernoffCoefficient(0.5,df,sigma1,sigma2)
} 


mvtAlphaChernoffCoefficient <- function(alpha,df,sigma1,sigma2)
{
dd=nrow(sigma1)
t=-(df+dd)/2
tt=-1/(2*t)
det1=det(sigma1)
det2=det(sigma2)
invSigma1=solve(sigma1)
invSigma2=solve(sigma2)

num=((alpha*det1^(tt))*invSigma1+((1-alpha)*det2^(tt))*invSigma2)
den=(alpha*det1^(tt)+(1-alpha)*det2^(tt))

sigmaa=solve(num  /den )

dalpha=det(sigmaa)

0.5*(alpha*det1^(tt)+(1-alpha)*det2^(tt))^(t)*(dalpha^0.5)
}





expeMvt<-function(dim,lambda,n)
{
#central mvt
mu=rep(0,dim)
df=6
sigma1=diag(dim)
sigma2=diag(x = lambda, dim,dim)

x1=rmvt(n=n,sigma=sigma1,df)
x2=rmvt(n=n,sigma=sigma2,df)

misclassified=0

for (i in 1:n) {

d1=dmvt(x1[i,1:dim],mu,sigma1,df,log=FALSE)
d2=dmvt(x1[i,1:dim],mu,sigma2,df,log=FALSE)

 
# x1 comes from D1 but is classified as D2
if (d1<=d2) {misclassified=misclassified+1}

d1=dmvt(x2[i,1:dim],mu,sigma1,df,log=FALSE)
d2=dmvt(x2[i,1:dim],mu,sigma2,df,log=FALSE)

if (d2<=d1) {misclassified=misclassified+1}
}
expePe=misclassified/(2*n)
expePe
}



# test function
testMvt <-function(dim,lambda)
{
n=10000

stoPe=mvtPe(df,sigma1,sigma2,n);
expePe=expeMvt(dim,lambda,n)
bhat=mvtBhattacharryaCoefficient(df,sigma1,sigma2)
cher=mvtChernoffCoefficient(df,sigma1,sigma2)

 
c(stoPe,expePe,cher[1],cher[2],bhat)
}

 

#
# Main body
#

set.seed(2013)


dim=3
lambda=4
res=testMvt(dim,lambda)
cat("Dim=",dim, " Lambda=",lambda, "Stochastic Pe=", res[1], " Pe=", res[2]," Chernoff=",res[3]," best exponent=", res[4], " Bhat=",res[5],"\n");
\end{lstlisting}

\end{document}